\documentclass[11pt]{article}

\usepackage{comment}

\usepackage{booktabs}

\usepackage{amsthm}
\usepackage{amsfonts,amsmath}
\usepackage{amssymb}
\usepackage{graphicx}
\usepackage{tikz}
\usepackage[noend,ruled]{algorithm2e}
\usepackage[textsize=scriptsize]{todonotes}

\newtheorem{definition}{Definition}
\newtheorem{example}{Example}

\newtheorem{theorem}{Theorem}

\newtheorem{corollary}{Corollary}
\newtheorem{observation}{Observation}

\newcommand{\np}{\ensuremath{\mathrm{NP}}}

\newcommand{\fpt}{\ensuremath{\mathrm{FPT}}}

\newcommand{\wone}{\ensuremath{\mathrm{W[1]}}}
\newcommand{\wtwo}{\ensuremath{\mathrm{W[2]}}}
\newcommand{\w}{\ensuremath{\mathrm{W}}}
\newcommand{\p}{\ensuremath{\mathrm{P}}}

\newcommand{\orderedsetof}[2]{\ensuremath{\{#1_1, #1_2, \ldots, #1_#2\}}}
\newcommand{\namedorderedsetof}[3]{\ensuremath{#1=\orderedsetof{#2}{#3}}}
\newcommand{\canddef}{\namedorderedsetof{C}{c}{m}}
\newcommand{\votersdef}{\namedorderedsetof{V}{v}{n}}

\newcommand{\calR}{\ensuremath{{{\mathcal{R}}}}}

\newcommand{\pos}{\ensuremath{{{\mathrm{pos}}}}}

\DeclareMathOperator{\dec}{dec}
\DeclareMathOperator{\req}{req}

\DeclareMathOperator{\price}{price}
\DeclareMathOperator{\dpos}{dpos}
\DeclareMathOperator{\score}{score}
\DeclareMathOperator{\maj}{maj}
\DeclareMathOperator{\preffunc}{pref}

\newcommand{\clique}{\textsc{Clique}}
\newcommand{\dsb}{\textsc{Destructive Shift Bribery}}

\usepackage[normalem]{ulem}

\title{Algorithms for Destructive Shift Bribery\thanks{Early version of this
paper was presented at the AAMAS 2016 conference.}}

\author{
Andrzej Kaczmarczyk\\
Technische Universit\"at Berlin\\
Berlin, Germany
\and
Piotr Faliszewski\\
AGH University\\
Krak\'{o}w, Poland
}

\begin{document}

\maketitle

\begin{abstract}
  We study the complexity of \textsc{Destructive Shift Bribery}. In this
  problem, we are given an election with a set of candidates and a set of voters
  (each ranking the candidates from the best to the worst), a despised candidate
  $d$, a budget $B$, and prices for shifting $d$ back in the voters' rankings.
  The goal is to ensure that $d$ is not a winner of the election. We show that
  this problem is polynomial-time solvable for scoring protocols (encoded in
  unary), the Bucklin and Simplified Bucklin rules, and the Maximin rule, but is
  $\np$-hard for the Copeland rule. This stands in contrast to the results for
  the constructive setting (known from the literature), for which the problem is
  polynomial-time solvable for $k$-Approval family of rules, but is $\np$-hard
  for the Borda, Copeland, and Maximin rules. We complement the analysis of the
  Copeland rule showing \w-hardness for the parameterization by the budget
  value, and by the number of affected voters. We prove that the problem is
  \w-hard when parameterized by the number of voters even for unit prices. From
  the positive perspective we provide an efficient algorithm for solving the
  problem parameterized by the combined parameter the number of candidates and
  the maximum bribery price (alternatively the number of different bribery
  prices).
\end{abstract}



\sloppy

\section{Introduction}

We study the complexity of the destructive variant of the \textsc{Shift Bribery}
problem. We consider the family of all (unary encoded) scoring protocols
(including the Borda rule and all $k$-Approval rules) and the Bucklin, Simplified
Bucklin, Copeland, and Maximin rules.  It turns out that for all of them---except
for the Copeland rule---the problem can be solved in polynomial time. This
stands in sharp contrast to the constructive case, where the problem is
$\np$-hard~\cite{elk-fal-sli:c:swap-bribery} (and hard in the parameterized
sense~\cite{bre-che-fal-nic-nie:j:prices-matter}) for the Borda, Copeland, and
Maximin rules (however, \textsc{Shift Bribery} is in $\p$ for the $k$-Approval
family of rules and the Simplified Bucklin and Bucklin
rules~\cite{elk-fal-sch:j:fallback-shift}).

The \textsc{Shift Bribery} problem was introduced by Elkind et
al.~\cite{elk-fal-sli:c:swap-bribery} to model (a kind of) campaign
management problem in elections. The problem is as follows: We are
given an election, that is, a set of candidates and a set of voters that
rank the candidates from the most to the least desirable one, a
preferred candidate $p$, a budget $B$, and the costs for shifting $p$
forward in voters' rankings.  Our goal is to decide if it is possible to
ensure $p$'s victory by shifting $p$ forward, without exceeding the budget.
In this paper we study the destructive variant of the problem, where the
goal is to ensure that a given despised candidate $d$ does not win the
election, by shifting him or her back in voters' rankings (but, again,
each shift comes at a cost and we cannot exceed the given budget).

Studying destructive variants of election problems, where the goal is to change
the current winner (such as
manipulation~\cite{con-lan-san:j:when-hard-to-manipulate},
control~\cite{hem-hem-rot:j:destructive-control}, and
bribery~\cite{fal-hem-hem-rot:j:llull,mag-riv-she-wag:c:stv-bribery,car:c:margin-of-victory,xia:margin-of-victory}),
is a common practice in computational social choice, but our setting is somewhat
special. So far, all the destructive variants of the problems were defined by
changing the goal from ``ensure the victory of the preferred candidate'' to
``preclude the despised candidate from winning,'' but the set of available
actions remained unaffected (e.g., in both the constructive and the destructive
problem of control by adding voters, we can add some voters from a given pool of
voters and in both the constructive and the destructive bribery problem, we can
pay voters to change their votes in some way). In our case, we feel that it is
natural to divert from this practice and change the ability of ``shifting the
distinguished candidate forward'' to the ability of ``shifting the distinguished
candidate back''. Below we explain why.

If, when defining our destructive problem, we stuck with the ``ability
to shift forward,'' as in the \textsc{Constructive Shift Bribery} problem,
we would get the following problem: Ensure that a given despised
candidate does not win the election by shifting him or her forward in
some of the votes (without exceeding the budget). However, for a
monotone voting rule, if a candidate is already winning an election,
then shifting him or her forward certainly cannot preclude him or her
from winning. This would make our problem interesting only for a
relatively small set of nonmonotone rules.\footnote{Nonetheless, there are
 interesting nonmonotone rules, such as the single transferable rule (STV) and
 the Dodgson rule. It is also quite common for multiwinner rules to not be
 monotone (see the works of Elkind et
 al.~\cite{elk-fal-sko-sli:j:multiwinner-properties} and Faliszewski et
 al.~\cite{FSST16}), so studying this variant of the \textsc{Shift Bribery}
problem for multiwinner voting rules may be interesting.}
Further, the problem certainly would not be modeling what one would intuitively
think of as negative, destructive campaigning. Thus, while it certainly would be
interesting to study how to exploit nonmonotonicity of rules such as STV and
Dodgson (or various multiwinner rules) to preclude someone from winning, it
would not be the most practical way of defining \textsc{Destructive Shift
Bribery}.\medskip

\subsection{Related Work}
In recent years, \textsc{Constructive Shift Bribery} received quite
some attention.  The problem was defined by Elkind et
al.~\cite{elk-fal-sli:c:swap-bribery}, as a simplified variant of
\textsc{Swap Bribery} (which itself received some attention, for example, in
the works of Bredereck et al.~\cite{BFKNST17}, Dorn and
Schlotter~\cite{dor-sch:j:parameterized-swap-bribery}, Faliszewski et
al.~\cite{fal-rei-rot-sch:j:bucklin-fallback}, Knop et
al.~\cite{KKM17}, and papers regarding combinatorial domains, such as
those of Mattei et al.~\cite{mat-pin-ros-ven:j:cp-bribery} and Dorn
and Kr\"uger~\cite{dor-kru:j:cp-bribery}; importantly, Shiryaev et
al.~\cite{shir-yu-elkind:c:robust} studied a variant of
\textsc{Destructive Swap Bribery} and we comment on their work later).
Elkind et al.~\cite{elk-fal-sli:c:swap-bribery} have shown that
\textsc{Constructive Shift Bribery} is $\np$-hard for the Borda,
Copeland, and Maximin voting rules, but polynomial-time solvable for
the $k$-Approval family of rules. They also gave a $2$-approximation
algorithm for the case of Borda, which was later generalized to the
case of all scoring rules by Elkind and
Faliszewski~\cite{elk-fal:c:shift-bribery}. Chen et
al.~\cite{bre-che-fal-nic-nie:j:prices-matter} considered parametrized
complexity of \textsc{Constructive Shift Bribery}, and have shown a
varied set of results (in general, parametrization by the number of
positions by which the preferred candidate is shifted tends to lead to
FPT algorithms, parametrization by the number of affected voters tends
to lead to hardness results, and parametrization by the available
budget gives results between these two extremes). Recently, Bredereck
et al.~\cite{bre-fal-nie-tal:j:csb} studied the complexity of
\textsc{Combinatorial Shift Bribery}, where a single shift action can
affect several voters at a time. Their paper is quite related to ours,
because it was the first one in which shifting a candidate backward
was possible (albeit as a negative side effect, since the authors
studied the constructive setting).  Very recently, Maushagen et
al.~\cite{mau-nev-rot-sel:c:dsb} studied both constructive and
destructive shift bribery (using our model with backward shifts) for
the case of round-based rules such as STV (there referred as Hare),
Coombs, Baldwin, and Nanson.

\textsc{Shift Bribery} belongs to the family of bribery-related
problems, that were first studied by Faliszewski et
al.~\cite{fal-hem-hem:j:bribery} (see also the work of Hazon et
al.~\cite{haz-lin-kra:c:bribery} for a similar problem), and that
received significant attention within computational social choice
literature (see the survey of Faliszewski and
Rothe~\cite{fal-rot:b:bribery}). Briefly put, in the regular
\textsc{Bribery} problem the goal is to ensure that a given candidate
is a winner of an election by modifying---in an arbitrary way---up to
$k$ votes, where $k$ is part of the input.  Many types of bribery
problems were already studied, including---in addition to
\textsc{Swap} and \textsc{Shift Bribery}---\textsc{Support
  Bribery}~\cite{elk-fal-sch:j:fallback-shift}, \textsc{Extension
  Bribery}~\cite{bau-fal-lan-rot:c:lazy-voters,fal-rei-rot-sch:j:bucklin-fallback},
and others (e.g., in judgment
aggregation~\cite{bau-erd-rot:c:bribery-ja}, and in the setting of
voting in combinatorial
domains~\cite{mat-pin-ros-ven:j:cp-bribery,dor-kru:j:cp-bribery,mar-mau-pin-ros-ven:cp-bribery}). However,
from our point of view the most interesting variant of the problem is
\textsc{Destructive Bribery}, first studied by Faliszewski et
al.~\cite{fal-hem-hem-rot:j:llull} and also,  independently, by Magrino et
al.~\cite{mag-riv-she-wag:c:stv-bribery} and
Cary~\cite{car:c:margin-of-victory} under the name
\textsc{Margin of Victory} (this problem was  also studied by Xia~\cite{xia:margin-of-victory} and Dey and
Narahari~\cite{dey-nar:c:margin-of-victory}). The idea behind  \textsc{Margin of
  Victory} is that it can be very helpful in validating
election results: If it is possible to change the election result by
changing (bribing) relatively few votes, then one may suspect
that---possibly---the election was tampered with.  Bribery problems
are also related to lobbying
problems~\cite{chr-fel-ros-sli:j:lobbying,erd-fer-gol-mat-rei-rot:j:probabilistic-lobbying,neh:c:lobbying}.

A variant of \textsc{Destructive Swap Bribery} was studied by Shiryaev
et al.~\cite{shir-yu-elkind:c:robust} in their work on testing how
robust are winners of given elections. They presented an analysis of
the  case where every swap has a unit
price and they showed that the problem is easy for scoring protocols
and for the Condorcet rule.\footnote{The authors' definition of the
  Condorcet rule slightly differs from a standard one usually seen in
  the literature. They assume that if there is no unique Condorcet
  winner the rule returns the whole set of candidates instead of an
  empty set (well established as a return value for this case in the
  literature.)}  This work is very closely related to ours, but the
definition of the problem is somewhat different (more general types of
swaps but less general price functions). Very recently Bredereck et
al.~\cite{BFKNST17} studied an analogous setting (\textsc{Destructive
  Swap Bribery} with unit prices) for the case of multiwinner voting
rules.


\subsection{Our Contribution.}
We believe that \dsb{} is worth studying for three main reasons.
First, it simply is a natural variant of the \textsc{Constructive
  Shift Bribery} problem and as the constructive variant received
significant attention, we feel that it is interesting to know how the
destructive variant behaves (indeed, the work of Maushagen et
al.~\cite{mau-nev-rot-sel:c:dsb} is a sign that the destructive
setting also attracts attention).  Second, it models natural negative
campaigning actions, aimed at decreasing the popularity of a given
candidate.  Third, it serves a similar purpose as the \textsc{Margin
  of Victory}
problem~\cite{mag-riv-she-wag:c:stv-bribery,xia:margin-of-victory}: If
it is possible to preclude a given candidate from winning through a
low-cost destructive shift bribery, then it can be taken as a signal
that the election might have been tampered with, or that some agent performed a
possibly illegal form of campaigning.
Below we summarize the main contributions of this paper:
\begin{enumerate}
\item We define the \textsc{Destructive Shift Bribery} problem and
  justify why a definition that diverts from the usual way of defining
  destructive election problems is appropriate in this case.

\item We show that \dsb{} is a
  significantly easier problem than \textsc{constructive shift bribery}. To
  this end, we show polynomial time algorithms for \dsb{} for all scoring
  rules with unary encoded scores (including the Borda
  rule and $k$-Approval family of rules), Simplified Bucklin,
  Bucklin, and Maximin.

\item We show that in spite of our easiness results, there still are
  voting rules for which the problem is computationally hard. We
  exemplify this by proving $\np$-hardness and $\w[1]$-hardness (for
  several parameters) for the case of Copeland$^\alpha$ family of
  rules.
\end{enumerate}


The paper is organized as follows.  In Section~\ref{sec:prelim}, we
formally define elections, present the voting rules used in the paper,
define our problem, and briefly review necessary notions regarding
parametrized complexity. In Section~\ref{sec:results}, we present our
results, with one subsection for each of the studied rules.  We
conclude in Section~\ref{sec:conclusions} with an overview of our
work, tables of results, and suggestions for future research.

\section{Preliminaries}\label{sec:prelim}

For each positive integer $t$, by $[t]$ we mean the set $\{1, 2,
\ldots, t\}$.  We assume that the reader is familiar with standard
notions regarding algorithms and complexity theory, as presented in
the textbook of Papadimitriou~\cite{pap:b:complexity}.

\subsection{Elections and Election Rules}
An election is a pair $E = (C,V)$, where \canddef{} is
a set of candidates and \votersdef{} is a multiset of
voters. Each voter is associated with his or her preference order
$\succ _i$, that is, a strict ranking of the candidates from the best to
the worst (according to this voter).  For example, we may have
election $E = (C,V)$ with $C= \{ c_1, c_2, c_3\}$ and $V = \{v_1,v_2\}$,
where $v_1$ has preference order $v_1\colon c_1 \succ c_2 \succ c_3$.
If $c$ is a candidate and $v$ is a voter, we write $\pos_v(c)$ to
denote the position of $c$ in $v$'s ranking (e.g., in the preceding
example we would have $\pos_{v_1}(c_1)=1$).  Given an election $E =
(C,V)$ and two distinct candidates $c$ and $c'$, by $N_E(c,c')$
we mean the number of voters who prefer $c$ to $c'$.

An election rule $\mathcal{R}$ is a function that given an election $E
= (C,V)$ outputs a set $W \subseteq C$ of tied election winners
(typically, we expect to have a single winner, but due to symmetries
in the profile it is necessary to allow for the possibility of
ties). We use the unique-winner model, that is, we require a candidate to
be the only member of $\calR(E)$ to be considered $E$'s winner (see
the works of Obraztsova et
al.~\cite{obr-elk:c:random-ties-matter,obr-elk-haz:c:ties-matter} for
various other tie-breaking mechanisms and algorithmic consequences of
implementing them; indeed, there are situations where the choice of the
tie-breaking rule affects the complexity of election-related
problems).

We consider the following voting rules (for the description below, we
consider election $E = (C,V)$ with $m$ candidates; for each rule we
describe the way it computes candidates' scores, so that the
candidates with the highest score are the winners, unless explicitly stated otherwise):
\begin{description}
\item[Scoring protocols.] A scoring protocol is defined through vector
  $\alpha = (\alpha_1, \ldots, \alpha_m)$ of nonincreasing,
  nonnegative integers.  The $\alpha$-score of a candidate $c \in C$
  is defined as $\sum_{v \in V} \alpha_{\pos_v(c)}$. The most popular
  scoring protocols include the family of $k$-Approval rules (for each
  $k$, $k$-Approval scoring protocol is defined through a vector of $k$
  ones followed by zeros) and the Borda rule, defined by vector $(m-1,
  m-2, \ldots, 0)$. The $1$-Approval rule is known as Plurality.

\item[Bucklin and Simplified Bucklin.] A Bucklin winning round is the
  smallest value $\ell$ such that there is a candidate whose
  $\ell$-Approval score is greater or equal to $\frac{|V|}{2}+1$ (in
  other words, if there is a candidate ranked among top $\ell$
  positions by a strict majority of the voters). All candidates whose
  $\ell$-Approval score is at least this value are the winners under the
  Simplified Buckling rule.  The Bucklin score of a candidate is his
  or her $\ell$-Approval score, where $\ell$ is the Bucklin winning
  round of the election. The candidates with the highest Bucklin score
  are the Bucklin  winners (note that the set of Bucklin winners is a subset of
  the set of Simplified Bucklin winners).


\item[Copeland.] Let $\alpha$, $0 \leq \alpha \leq 1$, be a rational number. The
  Copeland$^\alpha$ score of a candidate $c$ is defined as:
  \begin{align*}
    |\{ d \in C \setminus \{c\} \colon N_E(c,d) >  N_E(d,c)\}|
    +\alpha| \{ d \in C \setminus \{c\} \colon N_E(c,d) =  N_E(d,c)\}|.
  \end{align*}  
  In other words, candidate $c$ receives one point for each candidate
  whom he or she defeats in their head-to-head contest (i.e.,
  for each candidate over whom
  $c$ is preferred by a majority of the voters) and $\alpha$ points for
  each candidate with whom $c$ ties their head-to-head contest.

\item[Maximin.] The Maximin score of candidate $c$ is defined as
  $\min_{d \in C \setminus \{c\}} N_E(c,d)$.
\end{description}

We write $\score_E(c)$ to denote the score of candidate $c$ in
election $E$ (the election rule will always be clear from the context).
\begin{definition}
 For election $E=(C,V)$, a Condorcet winner is a candidate $c \in C$ who defeats
 all the other candidates in head-to-head contests. We call an election rule
 $\calR$ Condorcet-consistent if, for every possible election $E$, it selects a
 Condorcet winner when one exists.
\end{definition}
Among all the rules we consider only Maximin and Copeland are
Condorcet-consistent. However, even though they both share this property, the
results we obtained show that they differ significantly when considered from our
problem's perspective.

\subsection{Destructive Shift Bribery}
The \textsc{Destructive Shift Bribery} problem for a given election
rule $\calR$ is defined as follows. We are given an election $E =
(C,V)$, a despised candidate $d \in C$ (typically, the current
election winner), a budget $B$ (a nonnegative integer), and the prices
for shifting $d$ backward for each of the voters (see below). The goal
is to ensure that $d$ is not the unique $\calR$-winner of the election
by shifting him or her backward in the voters' preference orders
without exceeding the budget.

We model the ``prices for shifting $d$ backward'' as destructive
shift-bribery price functions. Let us fix an election $E=(C,V)$, with
\canddef{}, \votersdef{}, and let the despised
candidate be $d$. Let $v$ be some voter and let $j = \pos_v(d)$.
Function $\rho \colon \mathbb{N} \rightarrow \mathbb{N} \cup
\{+\infty\}$ is a destructive shift-bribery price function for voter
$v$ if it satisfies the following conditions:
\begin{enumerate}
\item $\rho (0) = 0$,
\item for each two $i, i'$, $i<i' \leq m-j$, it holds that $\rho (i)
  \leq \rho(i')$ for $i<i' \leq m-j$, and
\item $\rho(i) = + \infty$ for $i>m-j$ 
\end{enumerate}
For each $i$, we interpret value $\rho (i)$ as the price of shifting
$d$ back by $i$ positions in $v$'s preference order.  Value $+ \infty$
is used for the cases where shifting $d$ back by $i$ positions
is impossible (due to $d$'s position in the vote). 
We assume that in
each instance, the price functions are encoded by simply listing their
values for all arguments for which they are not $+\infty$.

For our hardness results, we focus on the case of unit price
functions, where, for each voter, shifting the despised candidate
backward by $i$ positions costs $i$ units of the budget. Sometimes we
also consider all-or-nothing price functions, where for each voter $v$
there is a value $c_v$ (which can also be set to $\infty$) such that
the cost of shifting the despised candidate $i$ positions backward
always costs $c_v$, irrespective of $i$ (except for $i=0$, where the
cost is zero by definition).

\begin{example}
Let us consider an instance of \dsb{}
for the Borda rule with the following election:
\begin{align*}
  v_1 \colon & b \succ a \succ c \succ d, &
  v_2 \colon & d \succ b \succ a \succ c, &
  v_3 \colon & d \succ c \succ a \succ b, &
  v_4 \colon & d \succ a \succ b \succ c. 
\end{align*}
We take $d$ to be the despised candidate, and assume that we have unit
price functions. In this election, candidate $d$ wins with $9$ points,
and the second-best candidate, $a$, has $6$ points. However, if we
shifted $d$ two positions back in $v_4$'s preference order, then $d$'s
score would decrease to $7$ and $a$'s score would increase to
$7$. Thus, $d$ would no longer be the unique winner. In consequence,
for $B = 2$ we have a ``yes''-instance of \dsb{}.  On the other hand,
shifting $d$ back by one position only (in whichever vote) cannot
preclude $d$ from winning. So, for $B = 1$ we have a ``no''-instance.
\end{example}

\subsection{Parameterized Complexity}
In the theory of parameterized complexity, the goal is to study the
computational difficulty of problems with respect to both their input
length, as in classic computational complexity theory, and some
additional ``parameters.''  For example, for \dsb{} problem, the
parameters may be the numbers of voters or candidates in the input
election, the budget values, the maximum numbers of shifts available
for a bribery etc. (we take the parameters to be numbers). For a
parameter $k$, we say that a problem parametrized by $k$ is
fixed-parameter tractable (is in $\fpt$) if there is an algorithm that
solves it in time $f(k) \cdot|I|^{O(1)}$, where $|I|$ is the length of
the encoding of a given instance and $f$ is an arbitrary computable
function (that depends on the parameter value only). Intuitively, if
a problem is in $\fpt$ for some parameter $k$, then we can hope that
its instances where $k$ is small can be solved efficiently.

Parametrized complexity theory also offers a theory of intractability.
In particular, it is widely believed that if a problem is \wone{}-hard
with respect to some parameter $k$, then there is no $\fpt$ algorithm
for this problem for this parametrization.  The original definition of
the class $\wone$ is quite involved and it is currently common to
define the class by providing one of its complete problems ($\clique$
parametrized by the size of the clique is the most common example) and
the notion of a parametrized reduction. In our case the situation is
even simpler: All our $\wone$-hardness proofs give polynomial-time
many-one reductions from well-known $\wone$-hard problems and
guarantee that the values of the parameters in the reduced-to
instances depend only on the values of the parameter in the
reduced-from instances (in other words, our proofs do not use full
power of parametrized reductions).



We point readers interesting in more detailed treatments of
parametrized complexity theory to the textbooks of
Niedermeier~\cite{nie:b:invitation-fpt} and Cygan et
al.~\cite{cyg-fom-kow-lok-mar-pil-pil-sau:b:fpt}.

\section{Results}\label{sec:results}
In this section we present our main results. We show polynomial-time
algorithms for the $k$-Approval family of rules, the Borda rule, all
scoring protocols (provided that they are encoded in unary or that the
destructive shift bribery price functions are encoded in unary), the
Bucklin family of rules, and the Maximin rule.  For the
Copeland$^\alpha$ family of rules, we prove $\np$-hardness and several
results regarding its parameterized complexity.

As a warm-up, we start with an observation regarding an upper bound of
the complexity of the \dsb{} problem.
\begin{observation}
  \label{obs:np-membership}
  The \dsb{} problem is in \np{} for every voting rule for which winner
  determination is in \p{}.
\end{observation}
\begin{proof}
  Given an instance of the problem, we guess in which votes to shift
  the despised candidate and by how many positions. Then we check if
  the cost of these shifts does not exceed the budget and if
  implementing them ensures that the despised candidate is not the
  unique winner of the election.
\end{proof}

Every voting rule that we consider in this paper is polynomial-time
computable and, thus, Observation~\ref{obs:np-membership} guarantees
that the \dsb{} problem for each of our rules is in $\np$ (in
particular, all our $\np$-hardness proofs in fact show
$\np$-completeness).

\subsection{The k-Approval Family of Rules}
We start with the $k$-Approval family of rules. In this case, our
algorithm is very simple: If $d$ is the despised candidate then in
each vote we should either not shift $d$ at all or shift him or her
from one of the top $k$ positions (where each candidate receives a
single point) to the $(k+1)$-st position (where he or she would
receive no points), in consequence also shifting the candidate
previously at the $(k+1)$-st position one place forward (to receive a
point). Choosing which action to do for each particular voter is easy
via the following greedy/brute-force algorithm (our algorithm is based
on a similar idea as that of Elkind et
al.~\cite{elk-fal-sli:c:swap-bribery} for the constructive case).

\begin{theorem}\label{thm:k-approval}
  For each $k \in \mathbb{N}$, the \textsc{Destructive Shift Bribery}
  problem for the $k$-Approval rule is in $\p$.
\end{theorem}
\begin{proof}
  Let $E = (C,V)$ be the input election with \canddef{} and \votersdef{},
  let $d=c_1$ be the despised
  candidate, let $B$ be the budget, and let \orderedsetof{\rho}{n}
  be the destructive shift-bribery price functions for the voters.
  Our algorithm works as follows.

  For every candidate $c \in C \setminus \{d\}$, we test if it is
  possible to guarantee that the score of $c$ is at least as high as
  that of~$d$, by spending at most $B$ units of budget, as follows:
  \begin{enumerate}
  \item We partition the voters into three groups, $V_{d,c}$, $V_d$,
    and $V'$, such that: $V_{d,c}$ contains exactly the voters that
    rank $c$ on the $(k+1)$-st position and that rank $d$ above $c$,
    $V_{d}$ contains the remaining voters that rank $d$ among top $k$
    positions, and $V'$ contains the other remaining voters.
  \item We guess two numbers, $a$ and $b$, such that $|V_{d,c}|\leq a$
    and $|V_d|\leq b$.
  \item We pick $a$ voters from $V_{d,c}$ for whom shifting $d$ to the
    $(k+1)$-st position is least expensive and we pick $b$ voters from
    $V_{d}$ for whom shifting $d$ to the $(k+1)$-st position is least
    expensive. We shift $d$ to the $(k+1)$-st position in the chosen
    votes.
  \item If $d$ is not the unique winner in the resulting election and the total cost of performed shifts is smaller than or equal to budget $B$,
    then we accept.  Otherwise, we either try a different candidate
    $c$ or different values of $a$ and $b$. After trying all possible
    combinations, we reject.
  \end{enumerate}

  The algorithm runs in polynomial time: It
  requires trying at most $O(m)$ different candidates and $O(n^2)$
  different values of $a$ and $b$. All the other parts of the
  algorithm require polynomial time (in fact, a careful implementation
  can achieve running time $O(mn^2)$).

  To show correctness of the algorithm we start with an observation that it is
  never beneficial to shift $d$ below position $k+1$. Further, an optimal
  solution, after which some candidate $c$ has at least as high a score as $d$,
  can consist solely of actions that shift $d$ backward from one of the top $k$
  positions to the $(k+1)$-st one, in effect either promoting $c$ to the $k$-th
  position (shifts in voter group $V_{d,c}$), or promoting some other candidate
  (shifts in voter group $V_{d}$).  We guess how many shifts of each type to
  perform and execute the least costly ones.
\end{proof}

\subsection{The Borda Rule and All Scoring Rules}

Next we consider the Borda rule. We also obtain a polynomial-time
algorithm, but this time we resort to dynamic programming. After
proving Theorem~\ref{thm:borda} below, we will show that our algorithm
generalizes to all scoring protocols (provided that either the scores are
encoded in unary or the price functions are encoded in unary).

\begin{theorem}\label{thm:borda}
  The \textsc{Destructive Shift Bribery} problem for the Borda rule is
  in $\p$.
\end{theorem}
\begin{proof}
  Let $E = (C,V)$ be the input election with \canddef{} and \votersdef{},
  let $d=c_1$ be the despised
  candidate, let $B$ be the budget, and let \orderedsetof{\rho}{n}
  be the destructive shift-bribery price functions for the voters.  As
  in the case of the proof of Theorem~\ref{thm:k-approval}, we give an
  algorithm that first guesses some candidate $c \in C \setminus
  \{d\}$ and then checks if it is possible to ensure---by shifting $d$
  backward without exceeding budget $B$---that $c$ has at least as
  high a score as $d$. The algorithm performing this test is based on
  dynamic programming. 

  We fix $c$ to be the candidate that we want to have a score at least
  as high as $d$. 
  Further, for each $j \in [n]$ and $k \in [m]$, we set $A(j,k)$ to be
  $1$ if $v_j$ ranks~$c$ among first~$k$ positions below~$d$, and we
  set it to be $0$ otherwise.
  Finally, we write $s$ to denote the difference between the scores of
  $d$ and $c$, that is, $s = \score_E(d) - \score_E(c)$ (it must be that
  $s > 0$; otherwise we could accept immediately since $d$ would not
  be the unique winner of the election).

  For each $j \in [n]$ and each positive integer $k$, we define
  $f(j,k)$ to be the smallest cost for shifting $d$ backward in the
  preference orders of the voters from the set \orderedsetof{v}{j},
  so that, if $E'$ is the resulting election, it holds that:
  \[
     s - (\score_{E'}(d) - \score_{E'}(c)) \geq k.
  \]
  In other words, $f(j,k)$ is the lowest cost of shifting $d$
  backward in the preference orders of voters from the set \orderedsetof{v}{j}
  so that the relative score of $c$ with respect to $d$
  increases by at least $k$ points. Our goal is to compute $f(n,s)$;
  if it is at most $B$ then it means that we can ensure that $c$ has
  at least as high a score as $d$ by spending at most $B$ units of the
  budget.
%
  We make the following observation.
  \begin{observation}
    \label{obs:ShiftLong}
    Shifting candidate $d$ backward by some $k$ positions decreases
    the score of $d$ by $k$ points and, if $d$ passes $c$, increases
    the score of $c$ by one point. In effect, relative to $d$, $c$
    gains, respectively, $k$ or $k+1$ points.
  \end{observation}
  Now, we can express function $f$ as follows. We have that $f(0,0) =
  0$ and for each positive $k$ we have $f(k,0) = \infty$ (for
  technical reasons, whenever $k < 0$, we take $f(j,k) = 0$). For each
  $j \in [n]$ and $k \in [s]$, we have:
  \[
    f(j,k) = \min_{k' \leq k} f(j-1,k-(k'+A(j,k')) + \rho_j(k').
  \]
  To explain the formula, we observe that:
  \begin{enumerate}
  \item The minimum is taken over the value $k'$; $k'$ gives the number of
    positions by which we shift $d$ backward in vote $v_j$.
  \item When we shift $d$ backward by $k'$ positions, relative to $d$,
    candidate $c$ gains $k'+A(j,k')$ points.
  \item The cost of this shift is $\rho_j(k')$.
  \end{enumerate}
  Based on these observations, we conclude that the formula is
  correct. It is clear that using this formula and standard dynamic
  programming techniques, we can compute $f(n,s)$ in polynomial time
  with respect to $n$ and $m$.  This means that we can test in
  polynomial time, for a given candidate $c$, if it is possible to
  ensure that $c$'s score is at least as high as that of~$d$.

  Our algorithm for \dsb{} for Borda
  simply considers each candidate $c \in C \setminus \{d\}$ and tests
  if, within budget $B$, it is possible to ensure that $c$ has at
  least as high a score as $d$. If this test succeeds for some $c$, we
  accept.  Otherwise we reject.
\end{proof}

A careful inspection of the above proof shows that there is not much
in it that is specific to the Borda rule (as opposed to other scoring
protocols). Indeed, there are only the following two dependencies:
\begin{enumerate}
\item The fact that the difference between the scores of candidate $d$
  and candidate $c$ (value $s$) can be bounded by a polynomial of the
  number of voters and the number of candidates (dynamic programming
  requires us to store a number of values of the function $f$ that is
  proportional to $s$, so it is important that this value is
  polynomially bounded).

\item The way we compute the increase of the score of $c$, relative to
  $d$, in the recursive formula for $f(j,k)$.
\end{enumerate}

Since it is easy to modify the formula for $f(j,k)$ to work for an
arbitrary scoring protocol~$\alpha$, we have the following corollary
(the assumption about unary encoding of the input scoring protocol
ensures that the first point in the above list of dependencies is not
violated).

\begin{corollary}
  There exists an algorithm that, given as input a scoring protocol
  $\alpha$ (for~$m$~candidates) encoded in unary and an instance $I$
  of the \textsc{Destructive Shift Bribery} problem (with the same
  number $m$ of candidates), tests in polynomial time if $I$ is a
  ``yes''-instance for the voting rule defined by the scoring protocol
  $\alpha$.
\end{corollary}

What can we do if, in fact, our scoring protocol is impractical to
encode in unary (e.g., if our scoring protocol is of the form
$(2^{m-1}, 2^{m-2}, \ldots, 1)$)? In this case, it is easy to show the 
following result.

\begin{corollary}
  There exists an algorithm that, given as input a scoring protocol
  $\alpha$ (for $m$ candidates) and an instance $I$ of the
  \textsc{Destructive Shift Bribery} problem (with the same number $m$
  of candidates), with price functions encoded in unary, tests in
  polynomial time if $I$ is a ``yes''-instance for the voting rule
  defined by the scoring protocol $\alpha$.
\end{corollary}

To prove this result, we use the same argument as before, but now
function $f$ has slightly different arguments: $f(j,t)$ is the maximum
increase of the score of $c$, relative to $d$, that one can achieve by
spending at most $t$ units of budget (using such ``dual'' formulation
is standard for bribery problems and was applied, for example, by Faliszewski
et al.~\cite[Theorem~3.8]{fal-hem-hem:j:bribery} in the first paper
regarding the complexity of bribery problems).

On the other hand, if both the scoring protocol and the price functions are
encoded in binary, and the scoring protocol is part of the input, then the
problem is $\np$-complete by a reduction from the \textsc{Partition}
problem.

\begin{definition}
  In the \textsc{Partition} problem, we are given a sequence of
  positive integers and we ask if it can be partitioned into two
  subsequences whose elements sum up to the same value.
\end{definition}

\begin{theorem}
The \textsc{Destructive Shift Bribery} problem is $\np$-complete if both the
scoring protocol and the price functions are encoded in binary, and the scoring
protocol is part of the input.
\end{theorem}
\begin{proof}
  Consider an instance of the \textsc{Partition} problem with sequence
  $S=(s_1, s_2, \ldots, s_n)$ and let $s = \sum_{i=1}^ns_i$ be the sum
  of the elements from $S$.  Without loss of generality, we assume
  that for each $i \in [n-1]$ we have $s_i \geq s_{i+1}$. We also
  assume that $s$ is even and that $s_1<\frac{s}{2}$ (for $s_1=\frac{s}{2}$ an
  instance is polynomial-time solvable, for $s_1>\frac{s}{2}$ we get a
  ``no''-instance of \textsc{Partition}). We form an instance of
  \textsc{Destructive Shift Bribery} as follows:
  \begin{enumerate}
  \item We introduce candidates $d$, $p_1, \ldots, p_n$, and dummy
    candidates $\{c_i^j \mid i,j \in [n] \}$.  
  \item We form a scoring vector $\alpha = (\alpha_1, \alpha_2,
    \ldots, \alpha_{n^2+n+1})$ such that $\alpha_1=\frac{s}{2}$, for
    each $i \in [n]$ we have $\alpha_{i+1}=s_i$, and for all $i \geq
    n+2$ we have $\alpha_i=0$. By our assumptions regarding the
    sequence $S$, we have that $\alpha$ is a legal scoring vector.
  \item Our election consists of $n$ votes, each representing an
    element of $S$. For each element~$s_i$ we construct vote $v_i$ by
    placing candidate $p_i$ on the first position and candidate $d$ on
    the $(i+1)$-st position. We fill the remaining $n$ places among
    the top $n+2$ positions in vote $v_i$ with candidates from the set
    $\{c_i^j \colon j \in [n]\}$ in an arbitrary way. Then we complete
    vote $v_i$ by placing the remaining candidates on positions $n+3,
    n+4, \ldots$ in an arbitrary order.
  \item We set the budget $B$ to be $\frac{s}{2}$. For every voter
    $v_i$, we define bribery function $\rho_i$ such that $\rho_i(0)=0$,
    for each $t \in [n-i+1]$ we have $\rho_i(t)=s_i$, and
    $\rho_i(t)=B+1$ for all other possible values of $t$.
  \end{enumerate}
  We illustrate this construction in the example below.

  \begin{example}
    \label{ex:sp-NP}
    Consider a \textsc{Partition} instance with sequence
    $S=(5,4,2,2,1)$. Our reduction forms an election with candidate
    set $C=\{d\} \cup \{p_1, \ldots, p_5\} \cup \{c_i^j \mid i,j \in
    [5]\}$ and votes $v_1, \ldots, v_5$. The scoring protocol $\alpha$
    is $(7,5,4,2,2,1,0, \ldots,0)$ and budget is $B=7$. The preference
    orders of the voters are:
	\begin{align*}
	v_1\colon \qquad & p_1 \succ d \succ c_1^1 \succ c_1^2 \succ c_1^3 \succ
	c_1^4 \succ c_1^5 \succ \dots \\
	v_2\colon \qquad & p_2 \succ c_2^1 \succ d \succ c_2^2 \succ c_2^3 \succ
	c_2^4 \succ c_2^5 \succ \dots \\
	v_3\colon \qquad & p_3 \succ c_3^1 \succ c_3^2 \succ d \succ c_3^3 \succ
	c_3^4 \succ c_3^5 \succ \dots \\
	v_4\colon \qquad & p_4 \succ c_4^1 \succ c_4^2 \succ c_4^3 \succ d \succ
	c_4^4 \succ c_4^5 \succ \dots \\
	v_5\colon \qquad & p_5 \succ c_5^1 \succ c_5^2 \succ c_5^3 \succ c_5^4 \succ
	d \succ c_5^5 \succ \dots
	\end{align*}
	For every voter $v_i$, its bribery function $\rho_i$ has the following values:
	$\rho_i(0)=0$, $\rho_i(t)=s_i$ for $t \in [6-i]$, and $\rho_i(t)=8$ for all
	other possible values of $t$.
  \end{example}

  The unique winner of the election constructed by our reduction is
  $d$, with score $s$.  For each candidate $p_i$, we have
  $\score(p_i)=\frac{s}{2}$, and for all $i,j \in [n]$ we have
  $\score(c_i^j)<\frac{s}{2}$. Scores of the $p_i$ candidates cannot
  change due to shifting $d$ backward within budget, because either
  they are ranked ahead of $d$ or they are ranked too far behind $d$
  (so it is impossible to shift $d$ behind them within budget).
  Similarly, no dummy candidate can achieve a score larger than
  $\frac{s}{2}$.  Thus to prevent $d$ from winning the election, one
  has to decrease his or her score at least by $\frac{s}{2}$ (so that
  the $p_i$ candidates have at least as high scores as $d$ has).

  We observe that by bribing some voter $v_i$, we are not able to
  decrease candidate $d$'s score by a value grater than the number of
  units of budget spent.  Moreover, the budget is $\frac{s}{2}$, which
  is also the number of points by which $d$'s initial score has to be
  reduced. Hence, if we choose to bribe some voter $v_i$, then we have
  to shift $d$ backward to the $(n+2)$-nd position, decreasing his or
  her score by $s_i$ points at the same time. Together, these
  observations mean that in a successful bribery we have to reduce
  $d$'s score by a sum of some elements $s_1,\ldots, s_n$. Further,
  this sum has to equal exactly $\frac{s}{2}$. Therefore, there exists
  a solution to the created \textsc{Destructive Shift Bribery}
  instance if and only if there exists a solution to the initial
  \textsc{Partition} instance.

  The reduction is computable in polynomial time and (by
  Observation~\ref{obs:np-membership}) the \dsb{} problem belongs to
  \np{}, which means that the problem is \np{}-complete.
\end{proof}

\subsection {The Bucklin Family of Rules}
We continue our analysis of the \textsc{Destructive Shift Bribery} problem by
considering the Bucklin and Simplified Bucklin rules. We find that for both
rules our problem is polynomial-time solvable (via a dynamic programming
approach). Before presenting our solution, it is helpful to define some
additional notation. For election $E=(C,V)$, we denote the $k$-Approval score of
some candidate $c \in C$ by $\score^k(c)$. For some fixed $k$, we frequently use
the term \emph{$k$-Approval point} referring to any single point a candidate gets during the
computation of his or her $k$-Approval score. Consequently, for some fixed $k$,
a candidate may lose or gain a $k$-Approval point as an effect of a shift action
of some voter. We use $\maj(V)=\left\lfloor \frac{|V|}{2}\right\rfloor+1$ to denote the
strict majority threshold for voters $V$. 
\begin{theorem}\label{thm:bucklin}
  The \textsc{Destructive Shift Bribery} problem for the Bucklin rule is in $\p$.
\end{theorem}
\begin{proof}
  Let us fix budget $B$ and election $E=(C,V)$, where
  $C=\lbrace c_1, c_2, \ldots, c_m \rbrace$ and
  $V=\{v_1, v_2, \ldots, v_n\}$. We also have shift-bribery prices
  functions $\{\rho_1, \rho_2, \ldots, \rho_n\}$. Without loss of generality, we
  assume that $c_1=d$ is the despised candidate.

  To prove our theorem, we give an algorithm which first guesses
  candidate $c_i \neq d$ and the Bucklin winning round $k$ and, then,
  checks whether $c_i$ is able to preclude $d$ from being the unique
  winner of the election while ensuring that $k$ is the Bucklin winning
  round (or the Bucklin winning round is even earlier). We start by
  defining an appropriate helper function, then we show that we can solve our
  problem using this function, and eventually we show that this function is
  polynomial-time computable.

  For each candidate $c_i \in C \setminus \{d\}$ and each $k \in [m]$,
  we define the following function $f_i^k$. For each
  $w \in \{0\} \cup [n]$ and each $p,q,q' \in [m]$, we let
  $f_i^k(w,p,q,q')$ be the lowest cost of bribing at most first $w$
  voters so that $\score^k(c_i)=p$, $\score^k(d)=q$, and
  $\score^{k-1}(d)=q'$. In other words, the function gives the cost of
  setting up the scores of candidate $c_i$ and $d$ in round $k$ (and
  just before round $k$, for $d$).

  To check whether there is a successful bribery (i.e., one that ensures
  that $d$ is not a unique winner of the election) of cost at most
  $B$, it suffices to check if there is candidate $c_i$, $c_i \neq d$,
  round number $k \in [m]$, and values $p, q, q'$ such that:
  \begin{enumerate}
  \item\label{c1} $q'<\maj(V)$,
  \item\label{c2} $p \geq \maj(V)$, 
  \item\label{c3} $p \geq q$, and
  \end{enumerate}
  $f_i^k(n,p,q,q') \leq B$.  Condition~(\ref{c1}) guarantees that $d$
  does not win in any round prior to $k$, Condition~(\ref{c2}) ensures
  that $k$ is the largest possible Bucklin winning round, and
  Condition~(\ref{c3}) imposes that $d$ is not a unique winner if $k$
  is the Bucklin winning round (note that the actual winning round
  number might be smaller than $k$ but then $d$ certainly does not
  win).
  If such a set of values exists, then we accept. Otherwise we reject.

  The above algorithm requires computing $O(n^3)$ values of at most
  $O(m^2)$ functions $f_i^k$. Thus, to show that the algorithm runs in
  polynomial time, it suffices to show that the values of the $f_i^k$
  functions are polynomial-time computable, which we now do.

  We compute values $f_i^k(n,p,q,q')$ using dynamic programming.  For
  each $p,q,q' \in [n]$, $f_i^k(0,p,q,q') = 0$ if prior to any bribery
  we have $\score^k(c_i) = p$, $\score^k(d) = q$, and
  $\score^{k-1}(d) = q'$. Otherwise, we have
  $f_i^k(0,p,q,q') = \infty$.  To compute $f_i^k(w,p,q,q')$ for
  $w > 0$, we express it using a recursive formula.  To present this
  formula compactly, for each voter $v_w$ and each $k$, $1 \leq k<m$,
  we define:
  \begin{enumerate}
  \item Function $L_w^k$ such that for each $j \in \{0\} \cup [m]$ we have
    $L_w^k(j) = 1$ if candidate $d$ is ranked below the $k$-th
    position after he or she is shifted backward by $j$ positions, and
    we have $L_w^k(j) = 0$ otherwise.

  \item Function $G_w^k$ such that for each $i,j \in [m]$, we have
    $G_w^k(i,j) = 1$ if after shifting candidate $d$ backward by $j$ positions,
    candidate $c_i$ is ranked among top $k$ positions, and
    we have $G_w^k(i,j) = 0$ otherwise.
  \end{enumerate}
  Using this notation, we express function $f_i^k$ as follows:
  \begin{align*}
    f_i^k(w,p,q,q')= & \min_{j \in \{0\} \cup [k]} \lbrace f_i^k(w-1, p-G_w^k(i,j),
                       q-L_w^k(i,j), r - L_w^{k-1}(j))+\rho_w(j) \rbrace
  \end{align*}
  Intuitively, for given $k$ and $i$, the above formula checks all
  possible values of $j$,\footnote{For function $f_i^k$ it is never
    necessary to shift $d$ to position lower than $k+1$, which is why
    we consider $j \in \{0\} \cup [k]$.} for which we shift $d$ backward by $j$
  positions in vote $v_w$, and uses shifts in the preceding
  votes to fulfill the function's definition.  This formula, together
  with standard dynamic programming techniques, allows us to compute
  the values of function $f_i^k$ in polynomial time. This completes
  the proof.
\end{proof}

By slightly changing the function presented in the above proof and
using a similar algorithm, we obtain an algorithm for the Simplified
Bucklin rule.
\begin{theorem}
  The \textsc{Destructive Shift Bribery} problem for the Simplified
  Bucklin rule is in $\p$.
\end{theorem}
\begin{proof}
  We adjust the proof of Theorem~\ref{thm:bucklin} to match the case
  of the Simplified Bucklin rule. For this rule, the exact score of a
  candidate in the Bucklin winning round is not taken into account
  while determining the winners (provided that it is greater than the
  majority of the voters). Therefore, we can simplify function
  $f_i^k(w,p,q,q')$ from the proof of Theorem~\ref{thm:bucklin} by
  removing parameter~$q$. Consequently, we remove the constraint
  associated with parameter~$q$ while checking for a successful
  bribery. The rest of the proof remains the same, so \dsb{} for the
  Simplified Bucklin rule is in \p{} as well.
\end{proof}

\subsection{The Copeland Rule}
Let us now move on to the case of Copeland$^\alpha$ family of rules.
We show that irrespective of the choice of $\alpha$,
$0 \leq \alpha \leq 1$, the \textsc{Destructive Shift Bribery} problem
is $\np$-complete for Copeland$^\alpha$. We also show $\w[1]$-hardness
of the problem for parametrizations by the budget, the number of
voters, and the number of bribed voters.  On the other hand, an FPT
algorithm for the parametrization by the number of candidates was
given by Knop et al.~\cite{KKM17} (formally, they did not study
\dsb{}, but it is immediate to see that their proof extends to this
setting).  To show our hardness results, we use reductions from one of
the classic $\np$-complete problems, \clique{} (which is also
$\w[1]$-hard for the parameterization by the clique size).

\begin{definition}
 In the \clique{} problem we are given a graph and an integer $k$. We ask
 whether it is possible to find a $k$-clique, that is, a size-$k$ set of pairwise
 adjacent vertices, in the given graph.
\end{definition}


We use the following notation in the proofs in this section. By putting some
set $S$ in a preference order, we mean listing the
contents of the set in an arbitrary but fixed order. To denote the reversed
order we use $\overleftarrow{S}$ (e.g., for set $S=\{a,b,c\}$, writing
$S$ in a preference order could mean $b \succ c \succ a$; then, putting
$\overleftarrow{S}$ in a preference order would mean $a \succ c \succ b$).

\begin{theorem}\label{thm:copeland}
  For each (rational) value of $\alpha$, $0 \leq \alpha \leq 1$, the
  \textsc{Destructive Shift Bribery} problem for the Copeland$^\alpha$ rule is
  $\np$-complete and W[1]-hard for the parameterization by the budget value and
  for the parameterization by the number of affected of voters.
\end{theorem}
\begin{proof}
  Since winner determination for Copeland$^\alpha$ is in $\p$, by
  Observation~\ref{obs:np-membership} we see that our problem is in
  $\np$ and it remains to show its \np{}-hardness. We give a reduction
  from \clique{}.

  Let us fix some arbitrary value of $\alpha$ and an instance of
  \clique{} with a given graph $G=(V(G),E(G))$ and an integer $k$. We
  construct an instance of \dsb{} with an election $E = (C,V)$, unit
  bribery prices and budget $B=3{k \choose 2}$. Let the set of
  candidates $C$ be
  $\{d,p\} \cup L \cup L' \cup V(G) \cup E(G) \cup S$ (members of
  $V(G)$ and $E(G)$ are both vertices and edges in $G$ and
  corresponding candidates in our election).  Sets $L$, $L'$, and $S$
  contain dummy candidates, where $L$ and $L'$ consist of
  $|V| \cdot |E| \cdot B$ candidates each, and $S$ consists of
  ${k \choose 2} + k + 1$ candidates. We introduce the following
  voters:
  \begin{enumerate} 
  \item For each edge $\{u,v\}=e \in \bar{E}$, we introduce two voters with the following preference orders:
  \begin{align*}
  \text{1 vote} \colon& d \succ u \succ v \succ e \succ L \succ L' \succ p \succ
  E(G) \setminus\{e\} \succ V(G) \setminus\{u,v\} \succ S \\
  \text{1 vote} \colon &  \overleftarrow{S} \succ \overleftarrow{V(G) \setminus \{u,v\}} \succ \overleftarrow{E(G) \setminus \{e\}} \succ p \succ \overleftarrow{L'} \succ \overleftarrow{L} \succ e \succ v \succ u \succ d.
  \end{align*}
  \item We introduce $2k-3$ voters with the following preference orders: 
    \begin{align*}
      k-2 \text{ votes} \colon & \qquad E(G) \succ d \succ L \succ S \succ p \succ
      L' \succ V(G),\\
      k-2 \text{ votes} \colon & \qquad p \succ L' \succ d \succ L \succ V(G) \succ
      S \succ E(G), \\
      1 \text{ vote} \colon & \qquad S \succ p \succ d \succ L' \succ V(G) \succ L
      \succ E(G).
    \end{align*}
  \item We introduce $6k^2$ voters with the following preference orders: 
    \begin{align*} 3k^2 \text{ votes} \colon & \qquad d \succ L \succ L' \succ
      C \setminus (\{d\} \cup L \cup L'), \\
      3k^2 \text{ votes} \colon & \qquad \overleftarrow{C \setminus (\{d\}
      	\cup L \cup L')} \succ d \succ L' \succ L. 
    \end{align*} 
  \end{enumerate}
  Note that the number of voters is odd, so the value of $\alpha$ is irrelevant.

  We present the scores of the candidates, prior to bribery, in
  Table~\ref{tbl:scores} (for some candidates we only provide upper
  bounds on their scores; to verify the values in the table, it is
  helpful to note that the preference orders of the pairs of voters in
  the first group are reverses of each other, and thus one can
  disregard them when calculating scores).  We see that candidate $d$
  is the winner.

  \begin{table}[t]
    \centering
    \begin{tabular}{c|l}
      \toprule 
      candidate(s) & score \\
      \midrule
      $d$ & $|V| + |E| + |L| + |L'| + |S|$ \\ 
      $p$ & $|V| + |E| + |L| + |L'| + 1$ \\ 
      $S$ & $ \leq |V| + |E| + |L'| + |S|$ \\ 
      $L$ & $ \leq |V| + |E| + |L| + |S| - 1$ \\ 
      $L'$ & $ \leq |V| + |E| + |L| + |L'| - 1$ \\
      $V$ & $\leq |E| + |V| - 1$ \\
      $E$ & $\leq |E| - 1$ \\
      \bottomrule
    \end{tabular}
    \caption{Scores of all the candidates in the constructed
    election. We indicate upper bounds by $\leq$ where scores relate to groups
    of candidates.} \label{tbl:scores}
  \end{table}


  By shifting $d$ backward without exceeding the budget, we can only
  change the outcome of head-to-head contests between $d$ and the
  candidates from $V(G)$ and $E(G)$. All the other candidates are
  either ranked too far away from $d$ or, as in the case for $L$ and
  $L'$, $d$ has too large advantage over them. For each candidate
  $v \in V(G)$, we have $N_E(d,v)-N_E(v,d) = 2k-3$, and for each
  $e \in E(G)$, we have $N_E(d,e)-N_E(e,d) = 1$. Thus, for $d$ to lose
  a head-to-head contest against a candidate $v \in V(G)$, $d$ has to
  be shifted behind $v$ in at least $k-1$ votes, and to lose a
  head-to-head contest against a candidate $e \in E(G)$, $d$ has to be
  shifted behind $e$ in at least one vote (note that the scores of
  candidates in $V(G)$ and $E(G)$ are so low that after a shift
  bribery that does not exceed the budget, neither of them can be a
  winner).




  Candidate $p$ has the second highest score in our election and we
  have $\score_E(d)-\score_E(p)={k \choose 2}+k$. Hence, to prevent
  $d$ from being a winner, we need to lower $d$'s score by at least
  ${k \choose 2}+k$. If our graph contains a clique of size $k$ that
  consists of edges $Q = \{e_1, \ldots, e_{k \choose 2}\}$, then for
  each voter from the first group that corresponds to one of these
  edges we shift $d$ backward by three positions.  This costs
  $3{k \choose 2}$ units of budget. Candidate $d$ loses one point for
  each edge from the clique and one point for each vertex from the
  clique (because $d$ passes each of them exactly $k-1$
  times). Altogether, $d$ loses ${k \choose 2}+k$ points and ceases to
  be the unique winner.

  For the other direction, assume that there is a shift bribery of
  cost at most $B = 3{k \choose 2}$ that ensures that $d$ is not the
  unique winner. By the observations from the previous paragraphs, we
  can assume that the bribery affects voters in the first group only
  (and for each pair of voters there, it affects the first one).  We
  claim that the bribery has to shift $d$ behind $k \choose 2$ edge
  candidates in the first group of voters. To see why this is the
  case, assume that it shifts $d$ behind $y={k \choose 2} - x$
  (distinct) edge candidates, where $x$ is a positive integer smaller
  or equal to ${k \choose 2}$. To make $d$ lose ${k \choose 2} + k$
  points, we need to shift $d$ backwards behind at least
  ${k \choose 2} + k - y=k+x$ vertex candidates at least $k-1$
  times. That means that altogether the number of unit shifts that we need to make is at least:
  \[
  \underbrace{3\left({k \choose 2} - x\right)}_{\text{passing edge
      candidates}} + \underbrace{(k+x)(k-1)}_{\substack{\text{total number of
      unit shifts} \\ \text{needed to pass vertex candidates}}} - \underbrace{2\left({k \choose
      2}-x\right)}_{\substack{\text{unit shifts for passing vertex candidates,}\\ \text{accounted for while passing edge candidates}}}.
  \]
  This value is equal to:
  \begin{align*}
     \textstyle {k \choose 2}-x + (k+x)(k-1) & =  \textstyle {k \choose 2} -x + k(k-1) + x(k-1) \\
     & = \textstyle 3{k \choose 2} -x + x(k-1) = \textstyle 3{k \choose 2} + x(k-2),
  \end{align*}
  which is greater than $B = 3{k \choose 2}$ for $k > 2$ and $x > 1$.
  Since we can assume that $k > 2$ without loss of generality, we see
  that a successful shift bribery that does not exceed the budget has
  to guarantee that $d$ passes exactly $k \choose 2$ edge candidates.
  One can verify that this leads to $d$ passing $k$ vertex candidates
  $k-1$ times each only if these edges form a size-$k$ clique.

  \clique{} is \wone{}-hard for the parameter $k$. As our budget is a
  function of $k$ and the number of affected voters is exactly
  $k \choose 2$, our reduction shows \wone{}-hardness with respect to
  the budget and with respect to the number of affected voters.
\end{proof}

The problem remains hard also for the parametrization by the number of
voters.

\begin{theorem}
  Parameterized by the number of voters, the \dsb{} problem is \wone{}-hard,
  even for the case of unit prices.
\end{theorem}
\begin{proof}
  \newcommand{\graphcands}{\ensuremath{G_{\text{c}}}}
  \newcommand{\vote}{\ensuremath{o}} We give a parameterized reduction
  from the \textsc{Multicolored Independent Set} problem, where we are
  given a graph $G=(V(G),E(G))$ with each vertex colored with one out
  of $h$ colors, and we ask if there is a size-$h$ subset of vertices
  $I \subseteq V(G)$ such that every vertex has a different color and
  for each two $u,v \in I$, there is no edge $\{u,v\}$ in the
  graph. We assume without loss of generality that the number of
  vertices of every color is the same, there are no edges between
  vertices of the same color, and there exists at least one vertex
  with non-zero degree among vertices of each color. Let $q$ be the
  number of vertices of each color.  We denote the maximum degree
  among vertices of graph $G$ by $\Delta$.  Let $\delta(v)$ be the
  degree of vertex $v \in V(G)$, and let $E(v)$ be the set of edges
  adjacent to a given vertex $v \in V(G)$. By $V(i)$ we mean the set
  of all vertices of color~$i$.

  For each color $i$, we introduce vertex candidates
  $V(i) = \{v^{i}_1, v^{i}_2, \dots, v^{i}_{q}\}$. We represent every
  edge in the input graph $G$ with one edge candidate. For each vertex
  candidate $v^{i}_j$, by $E_j^i$ we denote the set of edge candidates
  representing the edges incident to $v^i_j$ in graph $G$.  We write
  \graphcands{} to denote the set of all vertex and edge candidates.
  If for some $i \in [h]$ and $j \in [q]$, $|E^i_j|<\Delta$, we add
  the set of filler candidates $F^{i}_{j}$ of size
  $\Delta-\delta(v^{i}_j)$.  We write $F$ to denote the union of all
  the sets of filler candidates for all the vertices.  We introduce
  $h+3$ sets of dummy candidates with $t = hq(\Delta+1)$ elements each
  (note that $t > |V(G)|+|E(G)|$); namely, these sets are $D_1, \ldots, D_h$, $D'$, $D''$, and $D'''$.  We write $\mathcal{D}$ to denote
  $\bigcup_{i=1}^h D_i$ and $\mathcal{D}_{-i}$ to denote
  $\mathcal{D}\setminus D_i$.  Lastly, we add candidates $d$, $p$ and
  $q$. To specify voters in a compact form, for every color
  $i \in [h]$ we introduce partial preference order:
  \[
     P_i=v^i_1 \succ E^i_1 \succ F^i_1 \succ v^i_2 \succ \dots \succ  v^i_3 \succ \dots \succ v^i_q \succ E^i_q \succ F^i_q.  
  \]
  For a given $i \in [h]$, we write $P_{-i}$ to denote an arbitrary
  but fixed order over all the vertex, edge and filler candidates not
  ranked in $P_i$.
  The constructed election
  consists of the following votes:
  \begin{enumerate}
    \item For each color $i$ we construct votes:
    \begin{align*}
     \vote_i \colon & d \succ P_i \succ D_i \succ \mathcal{D}_{-i} \succ P_{-i} \succ p \succ q
     \succ D' \succ D'' \succ D''',
      \\
      \vote'_i \colon & d \succ \overleftarrow{P_i} \succ \overleftarrow{D_i} \succ
     \mathcal{D}_{-i} \succ P_{-i} \succ p \succ q \succ D' \succ D'' \succ D'''.
    \end{align*}
    We add votes $\bar{\vote}_i$ and $\bar{\vote}'_i$ with the reversed preference order
    of, respectively, $\vote_i$ and $\vote'_i$.
    \item We introduce seven votes as follows:  
    \begin{align*}
    u_1 \colon& p \succ F \succ \mathcal{D} \succ q \succ D' \succ \overleftarrow{D''} \succ \graphcands \succ D''' \succ d,\\
    u_2 \colon&d \succ \overleftarrow{D''} \succ  \overleftarrow{\graphcands} \succ D''' \succ  \overleftarrow{D'}\succ \mathcal{D} \succ q \succ  p \succ F,\\
    u_3 \colon& p  \succ q \succ d \succ \overleftarrow{D'} \succ D''' \succ \mathcal{D} \succ F \succ D'' \succ \graphcands,\\
    u_4 \colon & \graphcands \succ d \succ D'' \succ F \succ  \mathcal{D} \succ p \succ q \succ D''' \succ D', \\
    u_5 \colon & d \succ D''' \succ F \succ \graphcands \succ D'' \succ  \mathcal{D} \succ p \succ q \succ D', \\
    u_6 \colon & q \succ p \succ \mathcal{D} \succ D'' \succ d \succ \overleftarrow{D'''} \succ  D' \succ F \succ \graphcands,\\
    u_7 \colon & q \succ  D' \succ \graphcands \succ d \succ p \succ F \succ  \mathcal{D} \succ D'' \succ D'''.
    \end{align*}
  \end{enumerate}
  Our election admits unit prices, and the budget is $B=h(q + (q-1)\Delta)$.

  \begin{table}
    \centering
    \begin{tabular}{c|l}
      \toprule 
      candidate & score \\
      \midrule
      $d$ & $(h+3)t + |E| + |V| + |F| + 1$ \\ 
      $p$ & $(h+3)t + |F| + 1$ \\ 
      $q$ & $3t + |F| + |E| + |V| + 1$ \\ 
      $D_i, i \in [h] $ & $ \leq (h+3)t$ \\ 
      $D'$ & $ \leq t + |F| -1$ \\
      $D''$ & $ \leq 3t + |E| + |V| -1 $ \\
      $D'''$ & $ \leq 2t + |F| - 1$ \\
      $V$, $E$ & $\leq (h+2)t + |E| + |V|$ \\
      $F$ & $\leq (h+1)t + |E| + |V| + |F| -1$ \\
      \bottomrule
    \end{tabular}
    \caption{The scores of all of the candidates in the constructed
    election. We indicate upper bounds by $\leq$ where scores relate to groups
    of candidates.}
    \label{tbl:scores2}
  \end{table}

  We show the scores of the candidates (prior to bribery) in
  Table~\ref{tbl:scores2}.  For example, candidate $d$ wins
  head-to-head contests with all the candidates in $D'$, $D''$, $D''$,
  and $\mathcal{D}$, which accounts for the first $(h+3)t$ points. Candidate
  $d$ also wins with all the edge, vertex, and filler candidates,
  which gives additional $|E|+|V|+|F|$ points. Finally, $d$ wins
  with~$p$ but not with~$q$.

  \begin{table}[t]
    \centering
    \begin{tabular}{c|ccccccc}
      \toprule
      & $F$ & $V$ & $E$ & $\mathcal{D}$ & $D'$ & $D''$ & $D'''$ \\ 
      \midrule
      $d$ & $5$ & $1$ & $1$ & $3$ & $5$ & $3$ & $7$\\
      \bottomrule
    \end{tabular}
    \caption{For each set $X \in \{F,V.E.\mathcal{D}, D', D'', D'''\}$ and each $c \in X$,
      we report the difference between the number of voters that prefer $d$ to $c$ and those
      that prefer $c$ to $d$.}
    \label{tbl:vwone-wins}
  \end{table}

  Note that the budget value is smaller than the number of dummy
  candidates in each of the sets $D_1, \ldots, D_h, D', D''$, and
  $D'''$ (i.e., $B < t$). Moreover, in each vote in which $d$ can be
  shifted back, $d$ is either directly preferred to some set of dummy
  candidates or there are candidates from sets $F$, $E$ and $V$
  between $d$ and some set of dummy candidates. These two observations
  imply that, with a shift bribery of cost at most $B$, candidate $d$
  can be shifted back only behind candidates from sets $F$, $V$, $E$,
  $D'$, $D''$, $D'''$, and $\bigcup_{i \in [q]}D_i$.
  In~Table~\ref{tbl:vwone-wins} we present differences between the
  numbers of voters preferring $d$ to the candidates from these sets
  and the voters with the opposite preferences. Let us fix some candidate
  $c \in D' \cup D'' \cup D''' \cup \bigcup_{i \in [q]}D_i$. We note that there
  is no bribery of cost at most $B$ which makes $d$ lose a
  head-to-head contest with $c$.  If it were possible, then at least two voters
  would have to shift $d$ behind $c$. However, with a bribery of cost at most
  $B$, $d$ can pass candidate $c$ at most once.  This is a
  consequence of the fact that every time we shift $d$ behind $c$ in
  some vote, there exists only one other vote $v$ in which we can do
  this once more within cost $B$; however, the cost of such a shift in
  $v$ is always greater that the remaining budget. Let us now consider
  some candidate $c \in F$. For $d$ to lose the head-to-head contest
  against $c$, $d$ has to be shifted behind $c$ in at least three
  votes. However, there exist only two votes in which the price of a
  backward shift of $d$ behind $c$ is within the budget (e.g., for
  $c \in F^i_j$, $i \in [h]$, $j \in [q]$, the votes are $\vote_i$ and
  $\vote'_i$).
  
  The final conclusion from our observations is that to prevent $d$
  from being a unique winner of the election, with a shift bribery of
  cost at most $B$, one has to shift candidate $d$ behind every edge
  and vertex candidate at least once. Then candidate $d$ would no
  longer be the unique winner because $d$ and $p$ would tie (note that
  as $t>|E|+|V|$, no other candidate could threaten $d$). We claim
  that if one is able to shift $d$ back behind every edge and vertex
  candidate at least once without exceeding the budget, then the input
  graph $G$ has a multicolored independent set of size $h$. To show
  this, let us focus on some color $i$. Due to the budget constraint
  and votes' construction, we have to shift $d$ back behind every
  vertex $v \in V(i)$ bribing only voters $o_i$ and $o'_i$. To achieve
  this, we have to spend at least $q + (q-1)\Delta$ units of
  budget. We observe that it is also the upper bound because our
  argument holds for every color and we are constrained by the budget
  $B=h(q+ (q-1)\Delta)$. However, spending $q +(q-1)\Delta$ units of
  budget for some color $i$ does not allow us to shift $d$ back behind
  all the candidates in $\bigcup_{j \in [q]} E^i_j$. We have to shift
  $d$ back behind all vertex candidates of color $i$ and the only
  possibility to achieve this is to leave all candidate edges of one
  vertex candidate not passed by $d$. Let us call this vertex a
  selected vertex. Now, we can see that the only case where all edge
  candidates are, nonetheless, passed by candidate~$d$ is if the
  selected vertex candidates form a multicolored independent set. If
  this is not the case, than there exists at least one edge $e$
  connecting two selected vertices. Since edge candidates of selected
  vertices are not passed by $d$, the edge candidate corresponding to
  $e$ is never passed by $d$ and, so, $d$ is still the unique
  winner. On the other hand, one can also verify that if there exists
  a multicolored independent set, then one can always find a
  successful bribery.
  
  Since our reduction is a valid parameterized (indeed, a
  polynomial-time computable one) reduction where the number of voters
  is a function of the parameter $h$, and \textsc{Multicolored
    Independent Set} is \wone{}-hard for this parameter, we conclude
  that our problem is \wone{}-hard for the parameterization by the
  number of the voters.
\end{proof}

For the sake of completeness, we finish our analysis by mentioning one
more parameterized computational complexity result regarding \dsb{}
for Copeland$^\alpha$.

\begin{observation}
  Parameterized by the combined parameter number of voters and budget
  value, the \dsb{} problem is in \fpt{}. \dsb{} is in \fpt{} with
  respect to the parameter number of voters for all-or-nothing bribery
  functions (where, for each voter, there is a single price for every
  possible shift).
\end{observation}
\begin{proof}
  For unconstrained bribery functions one can guess which voters to
  bribe and how many budget units to use for each bribed voter.  The
  result holds for all-or-nothing prices functions because it suffices
  to guess the voters where we shift $d$ to the back.
\end{proof}

\subsection{The Maximin Rule}

In spite of the hardness results from the previous section, it
certainly is not the case that \textsc{Destructive Shift Bribery} is
hard for all Condorcet-consistent rules. We now show a polynomial time
algorithm for the case of Maximin.

\begin{theorem}
  \label{thm:maximin}
  The \textsc{Destructive Shift Bribery} problem for the Maximin rule
  is in $\p$.
\end{theorem}
\newcommand{\propnoun}{tightness}
\newcommand{\propadj}{tight}
\begin{proof}
  Let $E = (C,V)$ be the input election with
  $C = \{c_1, c_2, \ldots, c_m\}$ and $V = \{v_1, v_2, \ldots, v_n\}$, let
  $d \in C$ be the despised candidate, let $B$ be the budget, and let
  $\{\rho_1, \rho_2, \ldots, \rho_n\}$ be the destructive shift-bribery price
  functions for the voters. A solution to \dsb{} is a vector
  $S=(s_1, s_2, \ldots, s_n)$ of $n$ integers such that candidate $d$ does
  not win in the election resulting from bribing each voter
  $v_i \in V$ to shift $d$ back by $s_i$ positions. Naturally, $S$ is
  a solution if the price of all shifts it describes does not exceed
  the budget. Fix some solution $S$ and the resulting election
  $E' = (C, V')$, where $V'=\{v'_1, v'_2, \ldots, v'_n\}$. In election
  $E'$ there are always two important candidates: candidate $w$,
  whose score is at least as high as that of $d$ in $E'$, and candidate
  $t$ such that $\score_{E'}(d) = N_{E'}(d,t)$.
  Intuitively, $w$ is the candidate that ensures that $d$ is not a
  unique winner, and $t$ is the candidate that ``implements'' the
  score of $d$ in election $E'$. Note that it is possible that $w=t$.
  We call a solution $S = (s_1, s_2, \ldots, s_n)$ \emph{\propadj{}} if for
  every voter $v'_i$ such that $s_i \neq 0$, it holds that either
  $v'_i$ ranks $d$ just below $w$ or $v'_i$ ranks $d$ just below $t$.
  Intuitively, the solution is \propadj{} if we do not waste the
  budget to make unnecessary shifts that do not affect the relative
  order of $d$, $w$ and $t$. Note that we are able to modify any
  solution $S=(s_1, s_2, \ldots, s_n)$ so that it becomes \propadj{},
  by undoing, for every voter, as many shifts as possible without changing the relative
  positions of $d$, $w$, and $t$.  $S'$ is a solution because $d$, $w$, and $t$ have
  the same scores after applying $S'$ as after applying $S$. More so, the
  cost of $S'$ is at most as high as that of~$S$.
  Thus it suffices to focus on \propadj{} solutions.


  Our algorithm tries each pair of $t$ and $w$ and checks if there is
  a \propadj{} solution for them with cost at most~$B$. It accepts if
  so, and it rejects if there is no solution with cost at most $B$ for
  any choice of $t$ and $w$.  Since there are at most $O(m^2)$ pairs
  of candidates to try, it suffices to show a polynomial-time
  algorithm for computing the cheapest tight solution for given $t$
  and $w$. Below we focus on the case where $t \neq w$; the case where
  $t = w$ is analogous.

  Let us fix candidates $t$ and $w$, $t \neq w$.  For each two
  candidates $c_1, c_2$, we write $\preffunc(c_1,c_2)$ to denote the
  set of voters (from $V$) that prefer $c_1$ over $c_2$. Let
  $\price(v,c)$ be the cost of shifting candidate $d$ just below $c$
  in $v$'s preference order. Since we are interested in changing the
  relative positions of candidates $t$, $w$ and $d$, we focus on
  voters in the set
  $\preffunc(d,w) \cup \preffunc(d,t) = \{v''_1, v''_2, \cdots, v''_\ell \}$
  (i.e., voters who prefer $d$ to at least one of $w$ and $t$). We
  define function $f_{w,t}$ such that $f_{w,t}(j,x,y)$ is the lowest
  cost for shifting $d$ backward in the preference orders of the
  voters from $\{v''_1, v''_2, \ldots, v''_j\}$ in a way that ensures that
  $d$~moves~$x$~times below $w$ and $y$ times below $t$ (observe that
  $j, x, y \in \{0, 1, \ldots, \ell\}$). With function $f_{w,t}$, we can
  compute the cost of the cheapest tight solution for candidates $w$~and~$t$. It
  suffices to try all values of $f_{w,t}(\ell, x, y)$ for pairs $x, y \in \{0,
  1, \ldots, \ell\}$. For each pair it is easy to
  compute the scores that candidates $d$ and $w$ get after shifting
  $d$ back $x$~times behind~$w$~and~$y$ times behind $t$. If $w$ has
  at least as high score as $d$ and the cost is at most $B$, then we
  accept.
  
  We give a recursive formula for function $f_{w,t}$ that allows
  to compute the function's values in polynomial-time, by using
  standard dynamic programming techniques. For each
  $j,x,y \in \{0, 1, \ldots, \ell\}$, we have $f_{w,t}(j,0,0) = 0$. If
  $x > 0$ or $y > 0$, then $f_{w,t}(0,x,y) = \infty$ to indicate the
  impossibility of making shifts with no budget.\footnote{If the
    bribery price functions allow for shifting $d$ back at zero cost
    in some votes, we shift $d$ as much as possible at zero cost as a
    preprocessing step.} For other values of the arguments, a more
  involved discussion is necessary. Let us fix some values of
  $j \in [\ell]$, and $x,y \in \{0, 1, \ldots, \ell\}$. There are three
  cases to consider: \medskip

  \noindent
  \textbf{Case 1:}\quad If $v''_j \in \preffunc(d,w) \setminus
  \preffunc(d,t)$, that is, if $v''_j$ prefers $d$ to $w$, but not to
  $t$, then we have:
  \begin{displaymath}
    f_{w,t}(j,x,y) = 
    \min\left\{
      \begin{array}{l}
        f_{w,t}(j-1,x,y), \\
        f_{w,t}(j-1,x-1,y)+\price(v''_j, w)
      \end{array} \right\}. 
  \end{displaymath}
  To see the correctness of the formula, note that to achieve the fact
  that $d$ moves $x$ times behind $w$ and $y$ times below $t$ for the
  voters $v''_1, v''_2, \ldots, v''_j$, we either ensure that this happens
  already for the voters $v''_1, v''_2, \ldots, v''_{j-1}$ and leave $v''_j$
  intact, or we ensure that $d$~moves~${x-1}$~times below $w$~and~$y$
  times below $t$ for voters $v''_1, v''_2, \ldots, v''_{j-1}$ and shift $d$
  back behind $w$ in $v''_j$'s preference order (we omit such detailed
  descriptions below, but the general idea for each of the cases is
  the same).\medskip

  \noindent\textbf{Case 2:}\quad If $v''_j$ is in $\preffunc(d,t) \setminus \preffunc(d,w)$,
   then we have:
  \begin{displaymath}
    f_{w,t}(j,x,y) = 
    \min\left\{
      \begin{array}{l}
        f_{w,t}(j-1,x,y), \\
        f_{w,t}(j-1,x,y-1)+\price(v''_j, t)
      \end{array} \right\}. 
  \end{displaymath}

  \noindent\textbf{Case 3:}\quad If $v''_j$ is in $\preffunc(d,w) \cap \preffunc(d,t)$
  and $v''_j$ prefers $d$ to $w$ and $w$ to $t$, then we have:
  \begin{displaymath}
    f_{w,t}(j,x,y) = 
    \min\left\{
      \begin{array}{l}
        f_{w,t}(j-1,x,y), \\
        f_{w,t}(j-1,x-1,y)+\price(v''_j, w), \\
        f_{w,t}(j-1,x-1,y-1)+\price(v''_j, t)\\
      \end{array}
    \right\}.
  \end{displaymath}
  On the other hand, if $v''_j$ prefers $d$ to $t$ and $t$ to $w$ then
  we have:
  \begin{displaymath}
    f_{w,t}(j,x,y) = 
    \min\left\{
      \begin{array}{l}
        f_{w,t}(j-1,x,y), \\
        f_{w,t}(j-1,x,y-1)+\price(v''_j, t), \\
        f_{w,t}(j-1,x-1,y-1)+\price(v''_j, w)\\
      \end{array}
    \right\}.
  \end{displaymath}

  Given this discussion, using standard dynamic-programming approach it is
  possible to compute the values of functions $f_{w,t}$ in polynomial time. For the case of $w
  = t$, one has to slightly adapt all possible cases when computing function
  $f_{w,t}$.
\end{proof}

It is interesting to ask what feature of the Maximin rule---as opposed
to the Copeland rule---leads to the fact that \textsc{Destructive
  Shift Bribery} is polynomial-time solvable. We believe that the
reason is that it is safe to focus on a small number of candidates
(the candidates $w$ and $t$). For Copeland elections, on the other
hand, one has to keep track of all the candidates that the despised
candidate passes, because each such pass could, in effect, decrease
the despised candidate's score. (Interestingly, the same is true for
the Borda rule---each time the despised candidate passes a candidate,
the despised candidate's score decreases. However, in the case of
Borda this process is unconditional, whereas in the case of Copeland's
rule, the decrease may or may not happen, depending on shifts in other
preference orders).

\section{Conclusions}\label{sec:conclusions}

We have introduced and studied a destructive variant of the
\textsc{Shift Bribery} problem. In our problem, we ask if it is
possible to preclude a given candidate from being the winner of an
election by shifting this candidate backward in some of the votes (at
a given cost, within a given budget), whereas in the constructive
variant of the problem one asks if it is possible to ensure a given
candidate's victory by shifting him or her forward. 


\begin{table}
\begin{center}
\begin{tabular}{r|cc}
  \toprule
  & \textsc{Constructive} &  \textsc{Destructive} \\
  Election rule\phantom{$^\alpha$} & \textsc{Shift Bribery} & \textsc{Shift Bribery} \\
  \midrule
  Plurality\phantom{$^\alpha$}     &   $\p$  & $\p$  \\
  $k$-Approval\phantom{$^\alpha$}  &   $\p$  & $\p$  \\
  Borda\phantom{$^\alpha$}         &   $\np$-com.  & $\p$  \\
  Maximin\phantom{$^\alpha$}       &   $\np$-com.  & $\p$  \\
  Copeland$^\alpha$&   $\np$-com.  & $\np$-com.  \\
  \bottomrule
\end{tabular}
\end{center}
\caption{\label{tab:results}The complexity of \textsc{Shift Bribery} for
  various election rules. The results for the constructive case are due
  to Elkind et al.~\protect\cite{elk-fal-sli:c:swap-bribery} and the results regarding the destructive
  case are due to this paper.}
\end{table}

\begin{table}
\begin{center}
\begin{tabular}{r|cc}
  \toprule
  & \textsc{Constructive} &  \textsc{Destructive} \\
  Parameterization & \textsc{Shift Bribery} & \textsc{Shift Bribery} \\
  \midrule
  \#voters     &   $\wone{}$-h$^\spadesuit$ & $\wone{}$-h  \\
  \#voters $+$ budget    &   & $\fpt{}$  \\
  \#voters (all-or-nothing)    &   $\fpt{}^\spadesuit$  & $\fpt{}$  \\
  \#candidates  &   $\fpt{}^\blacklozenge$  & $\fpt{}^\blacklozenge$  \\
  budget       &   $\wtwo{}$-h$^\spadesuit$& $\wone{}$-h  \\
  \#affected voters&   $\wtwo{}$-h$^\spadesuit$  & $\wone{}$-h  \\
  \bottomrule
\end{tabular}
\end{center}
\caption{\label{tab:results-par}The complexity of \textsc{Shift Bribery} for
  Copeland$^\alpha$. The results marked with $^\blacklozenge$ are due to Knop et
  al.~\cite{KKM17} and the resuts marked with $^\spadesuit$
  follow from the work of Bredereck et
  al.~\cite{bre-che-fal-nic-nie:j:prices-matter}. The results regarding the
  destructive cases are due to this paper. The empty cell indicates that we are not
  aware of any work on this particular variant.}
\end{table}

We have shown that \textsc{Destructive Shift Bribery} is
polynomial-time solvable for the $k$-Approval family of rules (in
effect, including the Plurality rule), the Borda rule, all scoring
protocols (as long as either the protocol or the price functions can
be assumed to be encoded in unary) the Simplified Bucklin rule, the
Bucklin rule, and the Maximin rule. On the other hand, we have shown
that for each rational value of $\alpha$, the problem is
$\np$-complete for Copeland$^\alpha$. We have investigated the
problem's parameterized complexity in this case showing that it
remains hard for the case of small budgets and for the case of few
voters, even under unit price functions. However, the problem is in
FPT for the case of few candidates~\cite{KKM17}. We summarize our
results on general complexity in Table~\ref{tab:results}, whereas
Table~\ref{tab:results-par} contains the parameterized complexity
results for Copeland$^\alpha$.

Our work leads to several open questions. First, one could always
study more election rules. Second, one can analyze the robustness of
various election rules based on the number of backward shifts of the
winner needed to change their
outcome~\cite{shir-yu-elkind:c:robust,BFKNST17}. This direction can
also be seen as studying a more fine-grained extension of the
\textsc{Margin of Victory} problem. Third, it would be interesting to
perform an empirical test to measure how much we need to shift back
the election winner to change the result under various assumptions
regarding the voters' preference orders (and in real-life elections,
such as those collected in PrefLib~\cite{mat-wal:c:preflib}) to
complement the theoretical analysis mentioned in the second idea.

\section*{Acknowledgments}
We are grateful to AAMAS reviewers for their helpful comments. Piotr Faliszewski
was supported by the AGH University grant 11.11.230.124 (statutory research).
Andrzej Kaczmarczyk was partially supported by the AGH University grant
11.11.230.124 (statutory research) and partially by the DFG project AFFA (BR
5207/1 and NI 369/15).

\bibliographystyle{abbrv}

\end{document}